\documentclass[a4wide]{article}

\usepackage{soul}
\usepackage{times}
\usepackage{url}
\usepackage[utf8]{inputenc}

\usepackage[small]{caption}
\usepackage{graphicx}
\usepackage{amsmath}
\usepackage{amsthm}
\usepackage{booktabs}
\usepackage{algorithm}
\usepackage{algorithmic}
\usepackage{latexsym}
\usepackage{xspace}
\usepackage{amssymb}
\usepackage{mathtools}
\usepackage{turnstile}
\usepackage{graphicx}
\usepackage{xcolor}

\newtheorem{example}{Example}
\newtheorem{definition}{Definition}
\newtheorem{proposition}{Proposition}
\newtheorem{corollary}{Corollary}
\newtheorem{theorem}{Theorem}
\newtheorem{remark}{Remark}

\newcommand{\rhodf}{\mbox{$\rho$df}\xspace}
\newcommand{\universe}{\ensuremath{\mathtt{uni}}\xspace}
\newcommand{\voc}{\ensuremath{\mathtt{voc}}\xspace}
\newcommand{\rhodfbot}{\ensuremath{\rho df_\bot}}
\newcommand{\rhodfbotneg}{\ensuremath{\rho df_\bot^\neg}}

\newcommand{\AU}{\ensuremath{\mathbf{U}}\xspace}
\newcommand{\AB}{\ensuremath{\mathbf{B}}\xspace}
\newcommand{\AL}{\ensuremath{\mathbf{L}}\xspace}
\newcommand{\AUBL}{\ensuremath{\mathbf{UBL}}\xspace}
\newcommand{\AUL}{\ensuremath{\mathbf{UL}}\xspace}

\newcommand{\triple}[1]{(#1)}

\newcommand{\DR}{\ensuremath{{\sf R}}}
\newcommand{\DP}{\ensuremath{{\sf P}}}
\newcommand{\DC}{\ensuremath{{\sf C}}}
\newcommand{\DL}{\ensuremath{{\sf L}}}
\newcommand{\I}{\ensuremath{\mathcal{I}}\xspace}
\renewcommand{\int}[1]{{#1}^{\I} }
\newcommand{\intG}[1]{{#1}^{\I_G} }
\newcommand{\intP}[1]{{\mathsf{P}[\![#1]\!]}}
\newcommand{\intC}[1]{{\mathsf{C}[\![#1]\!]}}

\newcommand{\intPT}[1]{{\mathsf{P}^+[\![#1]\!]}}
\newcommand{\intCT}[1]{{\mathsf{C}^+[\![#1]\!]}}
\newcommand{\intPF}[1]{{\mathsf{P}^-[\![#1]\!]}}
\newcommand{\intCF}[1]{{\mathsf{C}^-[\![#1]\!]}}

\newcommand{\rdfsat}{\sdtstile{}{}~\!\!}
\newcommand{\rdfent}{\sdtstile{}{}~\!\!}
\newcommand{\deriv}{\sststile{}{}~\!\!}
\newcommand{\rdfsatbotneg}{\sdtstile{\neg,}{\bot}~\!\!}
\newcommand{\rdfentbotneg}{\sdtstile{\neg}{\bot}~\!\!}
\newcommand{\derivbotneg}{\sststile{\neg}{\bot}~\!\!}

\newcommand{\nturn}[1]{\raisebox{.2ex}{$\not$}{\ #1}}

\newcommand{\nrdfent}{\nturn{\sdtstile{}{}}}
\newcommand{\nderiv}{\nturn{\sststile{}{}}}

\newcommand{\clos}{\ensuremath{{\sf Cl}}}
\newcommand{\closbotneg}{\ensuremath{{\sf Cl_\bot^\neg}}}

\newcommand{\eee}{\ensuremath{\mathsf{e}}}
\newcommand{\tuple}[1]{\langle #1 \rangle }      
\newcommand{\spp}{\ensuremath{\mathsf{sp}}}
\newcommand{\subclass}{\ensuremath{\mathsf{sc}}}
\newcommand{\type}{\ensuremath{\mathsf{type}}}
\newcommand{\dom}{\ensuremath{\mathsf{dom}}}
\newcommand{\range}{\ensuremath{\mathsf{range}}}
\newcommand{\disjP}{\ensuremath{\bot_{\sf p}}}
\newcommand{\disjC}{\ensuremath{\bot_{\sf c}}}
\newcommand{\uniterm}{\ensuremath{\star}}

\newcommand{\up}[1]{\ensuremath{#1\!\!\uparrow}}
\newcommand{\down}[1]{\ensuremath{#1\!\!\downarrow}}

\newcommand{\impf}{\rightarrow}
\newcommand{\andf}{\land}
\renewcommand{\vec}{\mathbf}

\newcommand{\nd}{\noindent}
\newcommand{\ie}{\textit{i.e.},\xspace}
\newcommand{\eg}{\textit{e.g.},\xspace}
\newcommand{\wrt}{w.r.t.\xspace}
\renewcommand{\iff}{if and only if\xspace}
\newcommand{\cf}{\textit{cf.},\xspace}
\newcommand{\viz}{\textit{viz.},\xspace}
\newcommand{\ii}[1]{\mbox{$(#1)$}}

\newif\iftodo
\todotrue  

\newif\ifcomment
\commenttrue  

\title{A Minimal Deductive System for RDFS with Negative Statements}

\author{
\begin{tabular}{cc}
 Umberto Straccia$^1$ &  Giovanni Casini$^{1,2}$\\ \\
\multicolumn{2}{c}{$^1$ISTI - CNR, Pisa, Italy} \\
\multicolumn{2}{c}{$^2$CAIR - University of Cape Town, Cape Town, South Africa}
\end{tabular}
}

\begin{document}

\maketitle

\begin{abstract}
The  triple language RDFS is designed to represent and reason with \emph{positive} statements only (\eg~``antipyretics are drugs"). 
%

In this paper we show how to extend RDFS to express and reason with various forms of negative statements under the Open World Assumption (OWA). To  do so, we start from \rhodf, a minimal, but significant RDFS fragment that covers all essential features of RDFS, and then extend it to \rhodfbotneg, allowing express also statements such as
``radio therapies are \emph{non} drug treatments", ``Ebola \emph{has no} treatment", or "opioids and antipyretics are \emph{disjoint} classes". 
%
%
The main and, to the best of our knowledge, unique features of our proposal are:
\ii{i} \rhodfbotneg~remains syntactically a triple language by extending \rhodf~with new symbols with specific semantics and there is no need to revert to the reification method to represent negative triples; 
\ii{ii} the logic is defined in such a way that any RDFS reasoner/store may handle the new predicates as ordinary terms if it does not want to take account of the extra capabilities;
\ii{iii} despite negated statements, every \rhodfbotneg~knowledge base is satisfiable; 
\ii{iv} the \rhodfbotneg~entailment decision procedure is obtained from \rhodf~via  additional inference rules favouring a potential implementation; and
\ii{v} deciding entailment in \rhodfbotneg~ranges from P to NP.
\end{abstract}

\section{Introduction} \label{intro}

%
%
%


The \emph{Resource Description Framework} (RDF)\footnote{\url{http://www.w3.org/RDF/}} and its extension \emph{RDF Schema} (RDFS)\footnote{\url{https://www.w3.org/TR/rdf-schema/}} are both W3C standards, and nowadays quite popular knowledge representation languages. 
Essentially, a statement in RDF is a triple of the form $\triple{s,p,o}$, allowing to state that \emph{subject} $s$ is related to \emph{object} $o$ via the \emph{property} $p$. 
For instance,   
\[
\triple{\mathtt{fever}, \mathtt{hasTreatment}, \mathtt{paracetamol}}
\]
\nd is such a triple whose intended meaning is ``fever can be treated via paracetamol". RDFS is an extension of RDF providing mechanisms for describing groups of related terms and the relationships between these terms via a specific vocabulary of predicates. So, \eg~the RDFS triple 
\[
\triple{\mathtt{paracetamol}, \type, \mathtt{antipyretic}}
\]
\nd express that ``paracetamol \emph{is an} antipyretic" (here $\type$ is the predicate for class membership specification), while 
\[
\triple{\mathtt{antipyretic}, \subclass, \mathtt{drug}}
\]
\nd asserts that
 ``antipyretic \emph{is a subclass of} drug" ($\subclass$ is the predicate for sub-class specification).

While both languages have been  designed to represent and reason with \emph{positive} statements only, they can not properly deal with \emph{negative} statements such as
\begin{eqnarray}
&&\text{``opioids and antipyretics are \emph{disjoint} classes";} \label{op} \\
&&\text{``radio therapies are \emph{non} drug treatments"; and}  \label{rt} \\
&&\text{``Ebola \emph{has no} treatment".}  \label{eb}
\end{eqnarray}

\nd In particular, we may not infer that ``paracetamol \emph{is not} a treatment for Ebola". 
Such a statement could only be inferred with the major assumption that the Knowledge Base (KB) is complete — the so-called \emph{Closed-World Assumption} (CWA)~\cite{Reiter78}, which however is not realistic to be assumed in many cases. 
For instance, in medicine,  it is important to distinguish between knowing about the absence of a biochemical reaction between substances, and not knowing about its existence at all, which rises then the need for explicitly stating salient \emph{negative} statements (see, \eg~\cite{Arnaout21} for a recent work about it).
This is particularly true in the case in which the information about the represented world is assumed to be incomplete, — the so-called \emph{Open World Assumption} (OWA).

\paragraph{Contribution.}
In this paper we show how to extend RDFS to express and reason with various forms of negative statements under the OWA.
To do so, we start from \rhodf~\cite{Munoz09}, a minimal, but significant RDFS fragment that covers all essential features of RDFS, and then extend it to \rhodfbotneg, allowing express also negative statements via a specific expressions involving negated classes/properties, disjointness relationships, and no-value existence. So, for instance, the \rhodfbotneg~triple
\[
\triple{\mathtt{opioid}, \disjC, \mathtt{antipyretic}}
\]
\nd expresses $(\ref{op})$ ($\disjC$ is the vocabulary predicate for class disjointness specification),  
\[
\triple{\mathtt{radio therapy}, \subclass, \neg \mathtt{drugTreatment}}
\]
\nd expresses $(\ref{rt})$ (here, essentially  we introduce class complements via the $\neg$ operator), while 
\[
\triple{\mathtt{ebola}, \neg \mathtt{hasTreatment}, \uniterm_{\mathtt{treatment}}}
\]
\nd is meant to encode~$(\ref{eb})$
(here, besides property complement, we  also allow the $\uniterm_{c}$ operator, which is the place holder for an universally quantified variable over domain $c$).\footnote{We refer the reader to Table~\ref{tab:folmeaning} for an informal First-Order Logic (FOL) reading of some types of \rhodfbotneg~triples.}


The main and, to the best of our knowledge, unique features of our proposal are (\cf~related work below):
\ii{i} \rhodfbotneg~remains syntactically a triple language by extending \rhodf~with new symbols with specific semantics and there is no need to revert to the reification method to represent negative triples; 
\ii{ii} the logic is defined in such a way that any RDFS reasoner/store may handle the new predicates as ordinary terms if it does not want to take account of the extra capabilities;
\ii{iii} despite negated statements, every \rhodfbotneg~knowledge base is satisfiable, which is obtained via an intentional like four-valued semantics; 
\ii{iv} the \rhodfbotneg~entailment decision procedure is obtained from \rhodf~via  additional inference rules favouring a potential implementation; and
\ii{v} deciding entailment in \rhodfbotneg~ranges from P to NP.


\paragraph{Related Work.}
There have been various works in the past about extending RDFS with  negative statements, or applications that would like or require to have such a feature, which we briefly summarise below and indeed inspired our work.
%

In~\cite{Arnaout21} and related works~\cite{Arnaout20,Arnaout21a,Arnaout21b},  two types of negative statements are considered: 
\ii{i} grounded negative statements of the form $\neg \triple{s,p,o}$, with informal FOL reading $\neg p(s,o)$; and 
\ii{ii} universally negative statements of the form $\neg \exists x.\triple{s,p,x}$, meaning in FOL terms $\neg \exists x. p(s,x)$. 
The former type of triples have been proposed in~\cite{Analyti04} (and subsequent works, see below), while the latter has been addressed in~\cite{Darari15}.
In \cite{Arnaout21} essentially a statistical inference method is proposed to extract useful negative statements of this form, such as ``Leonardo DiCaprio has never been married’’ and ``United Kingdom is not the citizenship of Jimi Hendrix".\footnote{Optionally, triples may be annotated with a degree such as \eg~``The Sultan Resort has no parking facility to degree $0.97$". See \eg~\cite{Zimmermann12} for a general framework to deal with annotated triples.} It also publishes datasets\footnote{\url{https://github.com/HibaArnaout/usefulnegations}} that contain useful negative statements about entities in Wikidata.\footnote{https://www.wikidata.org}  Reasoning has not been addressed (and was not the focus of these works).
Both types of negative statements are covered by $\rhodfbotneg$ and, thus, our work is complementary to \cite{Arnaout21} in the sense that we describe how then to reason with such information.

In~\cite{Darari15}  the problem on how to express the non-existence of information is addressed, which has the form
$No(\{\triple{s_1,p_1,o_1}, \ldots, \triple{s_n,p_n,o_n})$, 
with informal FOL reading  $\neg \exists \vec{x}.(p_1(s_1,o_1) \land \ldots \land p_n(s_n,o_n))$, or  equivalently, $\forall \vec{x}.(\neg p_1(s_1,o_1) \lor \ldots \lor \neg p_n(s_n,o_n))$, where $\vec{x}$ are the variables occurring the triples. It shows how to represent it via the reification method and incorporate it into  SPARQL\footnote{\url{http://www.w3.org/TR/sparql11-query/}} query answering. Reasoning is not addressed however. 
We consider here only the case $n=1$  via the expression
$\triple{s, \neg p, \uniterm_c}$ as the general case $n \geq 2$ would introduce a disjunction, which we would like to avoid for computational reasons.



In~\cite{Analyti08} and related works~\cite{Analyti15,Analyti08a,Analyti09,Analyti11,Analyti13,Analyti05,Analyti04,Damasio10}
the authors deal with \emph{Extended} RDF (ERDF), a non-monotonic logic, where an ERDF ontology consists of two parts: an ERDF graph and an ERDF logic program. An ERDF graph allows \emph{negated} RDF triples of the form $\neg\triple{s,p,o}$,  informally in FOL terms  $\neg p(s,o)$, while in the body of rules all the classical  connectives  $\neg, \supset, \land, \lor, \forall, \exists$, plus the weak negation (negation-as-failure) $\sim$ are allowed. Various ``stable model" semantics have been proposed. From a computational complexity point view, decision problems in ERDF are non-polynomial~\cite{Analyti15}. E.g., deciding model existence and, thus, model existence is not guaranteed, ranges from NP to PSPACE, while query answering goes from co-NP to PSPACE, depending on the setting.\footnote{There are also many more works that use rule languages on top of RDFS, which however we are not going to discuss here (see, \eg~\cite{Casini20}.} In comparison, \rhodfbotneg~does not use a rule layer, the triple language is more expressive, model existence is guaranteed and the computational complexity ranges between P and NP. Of course, all inference rules for \rhodfbotneg~can be implemented in the rule layer of ERDF (and in Datalog in general).


Eventually, \cite{Casini20}  considers  $\rhodfbot$ on top of which to develop a non-monotonic RDFS logic based on Rational Closure~\cite{Lehmann92b}. $\rhodfbot$ extends $\rhodf$ allowing to express disjointness among (positive) classes and relations.

In summary, our work aims at putting all together within RDFS to deal with expressions of the form \eg~$(\ref{op})$--$(\ref{eb})$ in a generalised way.



We proceed as follows. As next we introduce the basic notions about \rhodf~we will rely on. Section~\ref{rhodfdisjneg} is the main part of this paper in which we extend \rhodf~to \rhodfbotneg. The paper concludes with a summary of the contributions and addresses some topics for future work.

\section{Preliminaries}

\nd For the sake of our purposes, we will rely on a minimal, but significant RDFS fragment, called 
$\rhodf$~\cite{Munoz09,Munoz07}, that covers the essential features of RDFS. In fact, $\rhodf$ may be considered as the logic behind RDFS and suffices to illustrate the main concepts and algorithms we will consider.
%
$\rhodf$ is defined as the following subset of the RDFS vocabulary:
\begin{equation} \label{vocrdf}
\rhodf = \{ \spp, \subclass, \type, \dom, \range\} \ .
\end{equation}

\nd Informally, 
\ii{i} $\triple{p,  \spp, q}$ means that property $p$ is a \emph{sub property} of property $q$;
\ii{ii}  $\triple{c, \subclass, d}$ means that class $c$ is a \emph{sub class} of class $d$; 
\ii{iii}  $\triple{a, \type, b}$ means that $a$ is of \emph{type}  $b$; 
\ii{iv}  $\triple{p, \dom, c}$ means that the \emph{domain} of property $p$ is $c$;
\ii{v}  $\triple{p, \range, c}$ means that the \emph{range} of property $p$ is $c$.


\paragraph{Syntax.} \label{sec:rdf-syntax} 
\nd Assume pairwise disjoint alphabets $\AU$ (\emph{RDF URI references}), $\AB$ (\emph{Blank nodes}), and $\AL$ (\emph{Literals}). We assume $\AU, \AB$, and $\AL$ fixed, and for simplicity we will denote unions of these sets simply concatenating their names. We call elements in $\AUBL$ \emph{terms} (denoted $a,b, \ldots, w$), and elements in $\AB$ \emph{variables} (denoted $x,y,z$).\footnote{All symbols may have upper or lower script.} A \emph{vocabulary} is a subset of $\AUL$ and we assume that $\AU$ contains the $\rhodf$ vocabulary (see Equation~\ref{vocrdf}).
%
%
%
A \emph{triple} is of the form
\[
\tau=\triple{s,p,o} \in \AUBL \times \AU \times \AUBL \ ,
\]
\nd where $s,o \notin \rhodf$.  We call $s$ the \emph{subject}, $p$ the \emph{predicate}, and $o$ the \emph{object}. A \emph{graph} (or \emph{RDF Knowledge Base}) $G$ is a set of triples $\tau$.
A subgraph is a subset of a graph. The \emph{universe} of $G$, denoted $\universe(G)$, is the set of terms in $\AUBL$ that occur in the triples of $G$. The \emph{vocabulary} of $G$, denoted by $\voc(G)$ is the set $\universe(G) \cap \AUL$. A graph is \emph{ground} if it has no blank nodes, \ie~variables.
%
A \emph{map} (or \emph{variable assignment}) is  as a function $\mu : \AUBL \to \AUBL$ preserving URIs and literals, \ie $\mu(t) = t$, for all $t \in \AUL$. Given a graph $G$, we define 
\[
\mu(G)= \{ \triple{\mu(s), \mu(p), \mu(o)} \mid \triple{s, p, o}\in G \} \ .
\]
\nd We speak of a map $\mu$ from $G_{1}$ to $G_{2}$, and write $\mu : G_{1} \to G_{2}$, if $\mu$ is such that $\mu(G_{1}) \subseteq G_{2}$.

\begin{example}[Running example]\label{exrdfs}
The following is a $\rhodf$~graph:\footnote{For ease of presentation, we use the terms paracetomol, antipyretic, morphine and opioid to mean paracetomol-, antipyretic-, morphine- and opioid-treatement, respectively. }
\begin{align*}
G =  \{ & 
      \triple{\mathtt{paracetamol}, \type, \mathtt{antipyretic}}, \\
    & \triple{\mathtt{antipyretic}, \subclass, \mathtt{drugTreatment}},  \\  
    &  \triple{\mathtt{morphine}, \type, \mathtt{opioid}},  
     \triple{\mathtt{opioid}, \subclass, \mathtt{drugTreatment}},  \\  
     & \triple{\mathtt{drugTreatment}, \subclass, \mathtt{treatment}},    \\
    & \triple{\mathtt{brainTumour}, \type, \mathtt{tumour}},    \\ 
    & \triple{\mathtt{hasDrugTreatment}, \spp, \mathtt{hasTreatment}},    \\
    & \triple{\mathtt{hasTreatment}, \dom, \mathtt{illness}},    \\
    & \triple{\mathtt{hasTreatment}, \range, \mathtt{treatment}},    \\
    & \triple{\mathtt{hasDrugTreatment}, \range, \mathtt{drugTreatment}}, \\
    &  \triple{\mathtt{fever}, \mathtt{hasDrugTreatment}, \mathtt{paracetamol}} \\
    &  \triple{\mathtt{brainTumour}, \mathtt{hasDrugTreatment}, \mathtt{morphine}} \  \}   \ .
\end{align*}
%
%
\end{example}
%
%
%
%
%
\paragraph{Semantics.} \label{sec:rdf-semantics} 
\nd A $\rhodf$ \emph{interpretation} $\I$ over a vocabulary $V$ is a tuple 
\[
\I =\tuple{\Delta_{\DR}, \Delta_{\DP}, \Delta_{\DC}, \Delta_{\DL}, \intP{\cdot}, \intC{\cdot}, \int{\cdot}} \ ,
\]
\nd where $\Delta_{\DR}, \Delta_{\DP}$, $\Delta_{\DC}, \Delta_{\DL}$ are the interpretation domains of $\I$, which are finite non-empty sets, and $\intP{\cdot}, \intC{\cdot}, \int{\cdot}$ are
the interpretation functions of $\I$. They have to satisfy: 

\begin{enumerate}
\item $\Delta_{\DR}$ are the resources; 
\item $\Delta_{\DP}$ are property names; 
\item $\Delta_{\DC} \subseteq \Delta_{\DR}$ are the classes; 
\item $\Delta_{\DL} \subseteq \Delta_{\DR}$ are the literal values and contains all the literals in $\AL \cap V$;
\item  $\intP{\cdot}$ is a function $\intP{\cdot}\colon \Delta_{\DP} \to 2^{\Delta_{\DR} \times \Delta_{\DR}}$;
\item  $\intC{\cdot}$  is a function  $\intC{\cdot}\colon \Delta_{\DC} \to 2^{\Delta_{\DR}}$;
\item $\int{\cdot}$ maps each $t \in \AUL \cap V$ into a value $\int{t} \in \Delta_{\DR} \cup \Delta_{\DP}$, where
  $\int{\cdot}$ is the identity for literals;  and
 \item $\int{\cdot}$ maps each variable $x \in \AB$ into a value $\int{x} \in \Delta_{\DR}$.

\end{enumerate}

\nd As next, for space reasons and without loosing the substantial ingredients, we illustrate the so-called~\emph{reflexive-relaxed} $\rhodf$ semantics~\cite[Definition 12]{Munoz09}, in which the predicates $\subclass$ and $\spp$ are \emph{not} assumed to be reflexive.
%
Informally, the notion entailment is defined using the idea of \emph{satisfaction} of a graph under certain interpretation. Intuitively a ground triple $\triple{s, p, o}$ in an RDF graph $G$ will be true under the interpretation $\I$ if $p$ is interpreted as a property name, $s$ and $o$ are interpreted as resources, and the interpretation of the pair $(s, o)$ belongs to the extension of the property assigned to $p$. Moreover, blank nodes, \ie~variables, work as existential variables. Intuitively the triple $\triple{x, p, o}$ with $x \in \AB$ will be true under $\I$ if $\I$ maps $x$ into a resource $s$ such that the pair $(s, o)$ belongs to the extension of the property assigned to $p$. Formally,

\begin{definition}[Model/Satisfaction/Entailment $\rdfsat$]\label{satisfaction}
A $\rhodf$ interpretation $\I$ is a \emph{model} of a  $\rhodf$ graph $G$, denoted $\I \rdfsat G$, \iff $\I$ is an interpretation over the vocabulary $\rhodf~\cup~\universe(G)$ such that:
\begin{description}\label{condRDF}
 \item[Simple:] \
 \begin{enumerate}
  \item for each $\triple{s, p, o} \in G$, $\int{p} \in\Delta_{\DP}$ and $(\int{s}, \int{o}) \in \intP{\int{p}}$;
 \end{enumerate}
 \item[Subproperty:] \
 \begin{enumerate}
 \item $\intP{\int{ \spp}}$ is transitive over $\Delta_{\DP}$;
 \item if $(p, q) \in \intP{\int{ \spp}}$ then $p, q \in \Delta_{\DP}$ and $\intP{p} \subseteq \intP{q}$; 
 \end{enumerate}
 \item[Subclass:] \
 \begin{enumerate}
 \item $\intP{\int{\subclass}}$ is transitive over $\Delta_{\DC}$; 
 \item if $(c, d) \in \intP{\int{\subclass}}$ then $c, d \in \Delta_{\DC}$ and $\intC{c} \subseteq \intC{d}$; 
 \end{enumerate}
 \item[Typing I:] \
 \begin{enumerate}
 \item $x \in \intC{c}$ \iff $(x,c) \in \intP{\int{\type}}$;
 \item if $(p, c) \in \intP{\int{\dom}}$ and $(x, y) \in \intP{p}$ then $x \in \intC{c}$;
 \item if $(p, c) \in \intP{\int{\range}}$ and $(x, y) \in \intP{p}$ then $y \in \intC{c}$;
 \end{enumerate}
 \item[Typing II:] \
 \begin{enumerate}
 \item for each $\eee \in \rhodf$, $\int{\eee} \in \Delta_{\DP}$;
 \item if $(p, c) \in \intP{\int{\dom}}$ then $p \in \Delta_{\DP}$ and $c \in \Delta_{\DC}$;
 \item if $(p, c) \in \intP{\int{\range}}$ then $p \in \Delta_{\DP}$ and $c \in \Delta_{\DC}$;
 \item if $(x, c) \in \intP{\int{\type}}$ then $c \in \Delta_{\DC}$.
 \end{enumerate}

\end{description}

\nd A graph $G$ is \emph{satisfiable} if it has a model $\I$. Moreover, given two graphs $G$ and $H$, we say that $G$ \emph{entails} $H$, denoted $G \rdfent H$, \iff every model of $G$ is also a model of $H$.
\end{definition}

\begin{example} \label{exrdfsI}
Consider Example~\ref{exrdfs}. Then it can be verified that
\[
G \rdfent \{ \triple{\mathtt{fever}, \mathtt{hasTreatment}, x},  \triple{x, \type, \mathtt{drugTreatment}} \} \ .
\]
%
%

\end{example}
%
%
\paragraph{Deductive System for $\rhodf$.} 
%
We recap the sound and complete deductive system for $\rhodf$~\cite{Munoz09}.  
In every rule, $A,B,C,D,E,X$ and $Y$ stand for meta-variables to be replaced by actual terms.  
An \emph{instantiation} of a rule is obtained by replacing all meta-variables with terms such that all triples  after the replacement are \rhodf~triples.
\begin{definition}[Deductive rules for $\rhodf$] \label{dedrhodf}
The \emph{deductive rules for $\rhodf$} are the following:
%
  \begin{enumerate}
  \item Simple: \\ [0.25em]
    \begin{tabular}{llll}
      $(a)$ & $\frac{G}{G'}$ for a map $\mu:G' \to G$ &  
      $(b)$ & $\frac{G}{G'}$ for  $G' \subseteq G$ 
    \end{tabular} 
    
  \item Subproperty: \\ [0.25em]
    \begin{tabular}{llll}
      $(a)$ & $\frac{\triple{A,  \spp, B},  \triple{B,  \spp, C}}{\triple{A,  \spp, C}}$ & $(b)$ & $\frac{\triple{D,  \spp, E},  \triple{X, D, Y}}{\triple{X, E, Y}}$ 
    \end{tabular} 
  \item Subclass: \\ [0.25em]
    \begin{tabular}{llll}
      $(a)$ & $\frac{\triple{A, \subclass, B},  \triple{B, \subclass, C}}{\triple{A, \subclass, C}}$ & $(b)$ & $\frac{\triple{A, \subclass, B},  \triple{X, \type, A}}{\triple{X, \type, B}}$ 
    \end{tabular} 
  \item Typing: \\ [0.25em]
    \begin{tabular}{llll}
      $(a)$ & $\frac{\triple{D, \dom, B},  \triple{X, D, Y}}{\triple{X, \type, B}}$ & $(b)$ & $\frac{\triple{D, \range, B},  \triple{X, D, Y}}{\triple{Y, \type, B}}$ 
    \end{tabular} 
  \item Implicit Typing: \\ [0.25em]
    \begin{tabular}{llll}
      $(a)$ & $\frac{\triple{A, \dom, B},  \triple{D,  \spp, A}, \triple{X, D, Y}}{\triple{X, \type, B}}$  \\\\ 
      $(b)$ & $\frac{\triple{A, \range, B},  \triple{D,  \spp, A}, \triple{X, D, Y}}{\triple{Y, \type, B}}$
    \end{tabular}
  \end{enumerate}
\end{definition}



\begin{definition}[Derivation $\deriv$]\label{def:derivatioGn}
Let $G$ and $H$ be $\rhodf$-graphs. $ G\deriv H$ \iff there exists a sequence of graphs $P_1,P_2,\ldots, P_k$ with $P_1=G$ and $P_k=H$ and for each $j$ ($2 \leq j \leq k$) one of the following cases hold:

\begin{itemize}
\item there is a map $\mu: P_j\rightarrow P_{j-1}$ (rule (1a));
\item $P_j \subseteq P_{j-1}$ (rule (1b));
\item there is an instantiation $R/R'$ of one of the rules (2)-(5), such that
$R \subseteq P_{j-1}$ and $P_j = P_{j-1} \cup R'$.
\end{itemize}

\nd \nd Such a sequence of graphs is called a proof of $G\deriv H$. Whenever
$G\deriv H$, we say that the graph $H$ is derived from the graph $G$. Each pair $(P_{j-1}, P_j)$, $1\leq j \leq k$ is called a step of the proof which is labeled by the respective instantiation $R/R'$ of the rule applied at the step.
\end{definition} 

\nd Please note that  if $G\deriv H$ then $H$ is indeed a graph.
Finally, the \emph{closure} of a graph $G$, denoted $\clos(G)$,  is defined as 
\[
\clos(G) = \{\tau \mid G \deriv^{*}~\tau\} \ ,
\]
\nd where $\deriv^{*}$ is as $\deriv$ except that rule $(1a)$ is excluded. 

\begin{example} \label{exrdfsIIA}
Consider Example~\ref{exrdfs}. Then it can be verified that
\[
\clos(G) \supseteq \{ \triple{\mathtt{morphine}, \type, \mathtt{drugTreatment}},
\triple{\mathtt{brainTumour}, \type, \mathtt{illness}} \} \ . 
\]
\end{example}

\nd The following proposition recaps salient results taken from~\cite{Gutierrez11,Munoz09,Horst05}

\begin{proposition} \label{proprdf}
Every $\rhodf$-graph is satisfiable. Moreover, let $G$ and $H$ be $\rhodf$-graphs. Then
\begin{enumerate}
\item $G \deriv H$ \iff $G \models H$;
\item if  $G \models H$ then there is a proof of $H$ from $G$ where rule $(1a)$ is used at most once and at the end;
\item the closure of $G$ is unique and $|\clos(G)| \in \Theta(|G|^2)$;
\item deciding $G\rdfent H$ is an NP-complete problem;
\item if $G$ is ground then $\clos(G)$ can be determined without using implicit typing rules $(5)$;
\item if $H$ is ground, then  $G\rdfent H$ \iff $H \subseteq \clos(G)$;
\item There is no triple $\tau$ such that $\emptyset \models \tau$.
\end{enumerate}
\end{proposition}
\begin{remark} \label{np}
Please note that: \ii{i} a proof of NP-completeness of point 4.~above can be found in \cite[Proposition 2.19]{Horst05} via a reduction of the 
$k$-clique problem encoding an undirected graph $G$ into a \rhodf~graph $G'$ (an edge $\tuple{v,w}$ is encoded via two triples $\triple{v,e,w}$ and $\triple{w,e,v}$) and $H'$ consists of the triples $\triple{x,e,y}$, where $x$ and $y$ are distinct variables of new set of $k$ blank nodes. Then, $G$ has a clique of size $\geq k$ iff $G' \models H'$, \ie~there is a map $\mu\colon H' \to G'$; and
\ii{ii} concerning the size of the closure, the lower bound is determined by triples 
$\triple{p_1, \spp, p_2}, \ldots, \triple{p_{n}, \spp, p_{n+1}}$, whose closure's size is $\Omega(n^2)$  via rule $(2a)$. The upper bound follows by an analysis of the rules, where the important point is the propagation of triples $\triple{s,p,o}$ via rule $(2b)$. This gives at most a quadratic upper bound (for triples with fixed predicate, the quadratic bound is trivial).
\end{remark}

\nd Please note that from Proposition~\ref{proprdf} it follows that deciding if, for two ground $\rhodf$-graphs $G$ and $H$, $G \models H$ can be done in time $O(|H||G|^2)$ by computing first the closure of $G$ and then check whether $H$ is in that closure. However, \cite{Munoz09} presents also an alternative method not requiring to compute the closure with a computational benefit as illustrated by the following proposition:
\begin{proposition} \label{propnlogn}
Let $G$ and $H$ be two ground \rhodf~graphs. Then deciding if $G \models H$ can be done in time $O(|H||G|\log |G|)$. The result hods also in case each triple in $H$ has at most one blank node.
\end{proposition}

\section{Extending \rhodf~with Negative Statements} \label{rhodfdisjneg}

\nd In this section, we show how to extend $\rhodf$ allowing to represent negative statements. 


\subsection{Syntax} \label{rhodfbotneg}

To start with, consider a new pair of predicates, $\disjC$ and $\disjP$, representing disjoint information: \eg 
\ii{i} $\triple{c,\disjC, d}$ indicates that the classes $c$ and $d$ are disjoint; analogously, 
\ii{ii} $\triple{p,\disjP, q}$ indicates that the properties $p$ and $q$ are disjoint.


We call $\rhodfbot$ the vocabulary obtained from $\rhodf$ by adding $\disjC$ and $\disjP$, that is, 
\begin{equation}\label{vocrhobot}
\rhodfbot = \rhodf \cup \{\disjC, \disjP\} \ .
\end{equation}

\nd  Like for $\rhodf$, we assume that $\AU$ contains the $\rhodfbot$ vocabulary.
%
%
%
Now we extend the alphabet $\AU$ in the following way:
\begin{enumerate}

\item for each (\emph{atomic}) resource $r \in \AU \setminus \rhodfbot$, we add to $\AU$ a new   \emph{negated resource} $\neg r$ of $r$. Let $\AU'$ be the resulting alphabet. We will use the convention that $\neg \neg r$ is $r$. Informally, $\neg r$ is intended to represent the complement of $r$. So, for instance, 
$\triple{\mathtt{paracetamol}, \type, \neg \mathtt{opioid}}$
may encode ``paracetamol is a non opioid treatmen)"; 

\item for each resource $c \in \AU' \setminus \rhodfbot$, we add to $\AU'$
a new resource of the form $\uniterm_c$.  Let $\AU''$ be the resulting alphabet.
Informally, \eg~a triple $\triple{s,p,\uniterm_c}$ represents an universal quantification on the third argument over instance of class $c$, \ie~$\triple{s,p,t}$ is true for all $t \in \AUL$ that are instances of the class $c$. For instance, $\triple{\mathtt{ebola}, \neg \mathtt{hasTreatment}, \uniterm_{\mathtt{treatment}}}$
may encode $(\ref{eb})$;\footnote{In the sense that ``none of the treatments are treatments for ebola"\label{footebola}}

\item finally, let $\AU$ be $\AU''$.

\end{enumerate}

\nd Now, the definition of \emph{$\rhodfbotneg$-triples} extends the one for $\rhodf$-triples in the following following way:

\begin{definition}[$\rhodfbotneg$-triple]\label{rdfbotneg}
A \emph{$\rhodfbotneg$-triple} is of the form
\begin{equation*} \label{rdhdfbotneg}
\tau=\triple{s,p,o} \in \AUBL \times \AU \times \AUBL \ ,
\end{equation*}

\nd where 
\begin{enumerate}
    \item $s,o \notin \rhodfbot$;
    \item $p$ is not of the form $\uniterm_c$;
    \item $s$ and $o$ can not be both of the form $\uniterm_c$;
    \item if $p \in \rhodfbot$ then neither $s$ nor $o$ are of the form $\uniterm_c$.
\end{enumerate}

\end{definition}

\nd In Table~\ref{tab:folmeaning}, to ease the reading, we provide an informal FOL reading of various additional (non exhaustive) types of triples supported in $\rhodfbotneg$. 

\begin{table}[t]
    \centering
{\small
    \begin{tabular}{|l|l|} \hline
    $\rhodfbotneg$ & FOL \\ \hline
     $\triple{s, \neg p,o}$    &  $\neg p(s,o)$ \\
     $\triple{s,\neg p,x}$    &  $\exists x. \neg p(s,x)$ \\
     $\triple{a,\type, \neg c}$    &  $\neg c(a)$ \\
     $\triple{c,\subclass, \neg d}$    &  $\forall x. c(x) \impf \neg d(x)$ \\
     $\triple{p,\dom, \neg c}$    &  $\forall x\forall y. p(x,y) \impf \neg c(x)$ \\
     $\triple{\neg p,\range, d}$    &  $\forall x\forall y. \neg p(x,y) \impf d(y)$ \\
     $\triple{c,\disjC, \neg d}$    &  $\forall x. c(x) \andf \neg d(x) \impf \bot$ \\
     $\triple{\neg p,\disjP, \ q}$    &  $\forall x\forall y. \neg p(x,y) \andf  q(x,y) \impf \bot$ \\
     $\triple{\uniterm_c, p, o}$    &  $\forall x. c(x) \impf p(x,o)$ \\
     $\triple{s, \neg p, \uniterm_c}$   & 
     $\forall y. c(y) \impf  \neg p(s,y)$  \ (\text{\ie}~ $\neg \exists y. c(y) \andf p(s,y))$\\  \hline
    \end{tabular}
}
   \caption{Informal FOL reading of some types of $\rhodfbotneg$-triples.}    \label{tab:folmeaning}

\end{table}

\begin{example}\label{ex_rdfbotnegI}
In the context of Example~\ref{exrdfs}, let us extend the graph $G$ with:
%
\begin{align*}
G \coloneqq \ & G \cup \{
      \triple{\mathtt{opioid}, \disjC, \mathtt{antipyretic}},  \\ & 
      \triple{\neg \mathtt{drugTreatment}, \subclass, \mathtt{treatment}},    \\ &
      \triple{\neg \mathtt{hasDrugTreatment}, \spp, \mathtt{hasTreatment}},    \\ &
      \triple{\neg \mathtt{hasDrugTreatment}, \range, \neg \mathtt{drugTreatment}},  \\ &
      \triple{\mathtt{brainTumour}, \neg \mathtt{hasDrugTreatment}, \mathtt{radioTherapy}}, \\ &
      \triple{\neg \mathtt{hasTreatment}, \dom,  \mathtt{illness}},  \\ &
      \triple{\neg \mathtt{hasTreatment}, \range,  \mathtt{treatment}},  \\ &
      \triple{\mathtt{ebola}, \neg \mathtt{hasTreatment}, \uniterm_{\mathtt{treatment}}} \  \}   \ .
\end{align*}
\end{example}

\subsection{Semantics} \label{rhodfbotnegsem}

\nd The semantics of \rhodfbotneg~has the following objectives:
\begin{enumerate}
\item we are going to accommodate the new constructs in such a way that the resulting deductive system will be as for $\rhodf$, plus some additional rules. In this way, any RDFS reasoner/store may handle the new triples as  ordinary triples if it does not want to take account of the extra inference capabilities;
    
\item the semantics has to be such that, despite introducing negative statements, all $\rhodfbotneg$ graphs have a canonical model (see Corollary~\ref{satrhodfbotneg} later on), and, thus, $\rhodfbotneg$ remains a \emph{paraconsistent} logic; and 
    

\item deciding entailment in \rhodfbotneg~still ranges from P to NP.
    
\end{enumerate}

\nd To do so, we will consider a \emph{four-valued} logic semantics~\cite{Belnap77a} variant of the semantics for $\rhodf$. Specifically, we will have \emph{positive extensions} of $\intP{}$ and $\intC{}$ (denoted $\intPT{}$ and $\intCT{}$, respectively) and \emph{negative extensions} of $\intP{}$ and $\intC{}$ (denoted $\intPF{}$ and $\intCF{}$, respectively).
Roughly, $\intCT{c}$ will denote the set of resources known \emph{to be} instances of class $c$, while $\intCF{c}$ will denote the set of resources known \emph{not to be} instances of class $c$ (for properties the case is similar). Note that positive and negative extensions need not to be the complement of each other: \eg~$r \notin \intCT{c}$ does not imply necessarily that $r \in \intCF{c}$ as $\intCF{c}$ will not enforced to be \eg~$\Delta_\DR \setminus \intCT{c}$.

The idea of having separate positive and negative extensions is not new at all and we may find already traces of it back in the mid 80s with the seminal work of  Patel-Schneider~\cite{Patel-Schneider85,Patel-Schneider86,Patel-Schneider87,Patel-Schneider88,Patel-Schneider89a} in which four-valued variants of \emph{Terminological Logics} (TLs), \viz~the so-called \emph{Description Logics} (DLs)~\cite{Baader07a} nowadays, have been proposed with the aim to obtain some gain from a computational complexity point of view. Later the works~\cite{Straccia97a,Straccia97,Straccia99a,Straccia00} have been inspired by the same idea, though also to model some sort of \emph{relevance entailment}, besides being paraconsistent. More recently, a similar idea has been considered also in the context of RDFS~\cite{Analyti04,Analyti05,Analyti13,Analyti15}, which is also the semantics we start from and are going to adapt and extend to meet the before mentioned objectives.

A $\rhodfbotneg$ \emph{interpretation} $\I$ over a vocabulary $V$ is now a tuple 
\[
\I =\tuple{\Delta_{\DR}, \Delta_{\DP}, \Delta_{\DC}, \Delta_{\DL}, \intPT{\cdot}, \intPF{\cdot}, \intCT{\cdot}, \intCF{\cdot},
\int{\cdot}} \ ,
\]
\nd where  $\Delta_{\DR}, \Delta_{\DP}$, $\Delta_{\DC}, \Delta_{\DL}$ are the interpretation domains of $\I$, which are finite non-empty sets, and $\intPT{\cdot}, \intPF{\cdot}, \intCT{\cdot}, \intCF{\cdot}, \int{\cdot}$ are
the interpretation functions of $\I$. They have to satisfy:


\begin{enumerate}
\item $\Delta_{\DR}$ are the resources; 

\item $\Delta_{\DP}$ are property names; 

\item $\Delta_{\DC} \subseteq \Delta_{\DR}$ are the classes; 

\item for each domain $\Delta_{\DR}, \Delta_{\DP}$ and $\Delta_{\DC}$, for each term $t$ in it, there is an unique designated complement term of $t$, denoted $\neg t$, in it;\footnote{As for $\AU$, we will  use the convention that $\neg \neg t$ is $t$.} 

\item $\Delta_{\DL} \subseteq \Delta_{\DR}$ are the literal values and contains all the literals in $\AL \cap V$;

\item $\intPT{\cdot}$ and $\intPF{\cdot}$ are functions $\Delta_{\DP} \to 2^{\Delta_{\DR} \times \Delta_{\DR}}$ such that $\intPT{\neg p} = \intPF{p}$, for each $p \in \Delta_{\DP}$;

\item $\intCT{\cdot}$ and $\intCF{\cdot}$   are functions  $\Delta_{\DC} \to 2^{\Delta_{\DR}}$ such that $\intCT{\neg c} = \intCF{c}$, for each $c \in \Delta_{\DC}$;

\item $\int{\cdot}$ maps each $t \in \AUL \cap V$, that is not of the form $\uniterm_c$, into a value $\int{t} \in \Delta_{\DR} \cup \Delta_{\DP}$, such that
$\int{(\neg t)} = \neg~\int{t}$ and
$\int{\cdot}$ is the identity for literals;  and

\item $\int{\cdot}$ maps each variable $x \in \AB$ into a value $\int{x} \in \Delta_{\DR}$.

\end{enumerate}

\nd In the following, we define 
\begin{eqnarray*}
\up{\intPT{p}} & = & \{x \in \Delta_{\DR} \mid (x,y) \in \intPT{p} \} \\
\down{\intPT{p}} & = & \{y \in \Delta_{\DR} \mid (x,y) \in \intPT{p} \} \\
\up{\intPF{p}} & = & \{x \in \Delta_{\DR} \mid (x,y) \in \intPF{p} \} \\
\down{\intPF{p}} & = & \{y \in \Delta_{\DR} \mid (x,y) \in \intPF{p} \} 
\end{eqnarray*}
\nd as the projections of the property extension functions $\intPT{}$ and $\intPF{}$ on the first and second argument, respectively.

Now, the model/satisfaction/entailment definitions for $\rhodf$ are generalised to $\rhodfbotneg$ as follows: 

\begin{definition}[Model/Satisfaction/Entailment $\rdfsatbotneg$]\label{satisfactionNew}
A $\rhodfbotneg$ interpretation $\I$ is a \emph{$\rhodfbotneg$-model} of a  $\rhodfbotneg$ graph $G$, 
denoted $\I \rdfsatbotneg G$, 
\iff $\I$ is a $\rhodfbotneg$-interpretation over the vocabulary $\rhodfbotneg~\cup~\universe(G)$ such that:

\begin{description}\label{condRDFA}
 \item[Simple:] \
 \begin{enumerate}
  \item if $\triple{s, p, o} \in G$ and neither $s$ nor $o$ are of the form $\uniterm_c$, then $\int{p} \in\Delta_{\DP}$ and $(\int{s}, \int{o}) \in \intPT{\int{p}}$;
  
  \item if $\triple{s, p, \uniterm_c} \in G$, then $\int{p} \in\Delta_{\DP}, \int{c} \in \Delta_{\DC}$ and $(\int{s}, y) \in \intPT{\int{p}}$, for all $y \in \intCT{\int{c}}$;
  
  \item if $\triple{\uniterm_c, p, s} \in G$, then $\int{p} \in\Delta_{\DP}, \int{c} \in \Delta_{\DC}$ and $(x,\int{s}) \in \intPT{\int{p}}$, for all $x \in \intCT{\int{c}}$;
  
  
  \item if $\triple{s, p, \uniterm_c} \in G$, then $\int{p} \in\Delta_{\DP}, \int{c} \in \Delta_{\DC}$ and $y \in \intCF{\int{c}}$, for all $(\int{s}, y) \in \intPF{\int{p}}$ ;
  
  \item if $\triple{\uniterm_c, p, s} \in G$, then $\int{p} \in\Delta_{\DP}, \int{c} \in \Delta_{\DC}$ and $x \in \intCF{\int{c}}$, for all   $(x,\int{s}) \in \intPF{\int{p}}$;
  
 \end{enumerate}
 
 \item[Subproperty:] \
 \begin{enumerate}
 \item $\intPT{\int{ \spp}}$ is transitive over $\Delta_{\DP}$;
 \item if $(p, q) \in \intPT{\int{\spp}}$ then $p, q \in \Delta_{\DP}$ and $\intPT{p} \subseteq \intPT{q}$;
 
 \item $(p, q) \in \intPT{\int{\spp}}$ \iff $(\neg q, \neg p) \in \intPT{\int{\spp}}$;
 
 \end{enumerate}
 
  \item[Subclass:] \
 \begin{enumerate}
 \item $\intPT{\int{\subclass}}$ is transitive over $\Delta_{\DC}$; 
 \item if $(c, d) \in \intPT{\int{\subclass}}$ then $c, d \in \Delta_{\DC}$ and $\intCT{c} \subseteq \intCT{d}$;
 \item $(c, d) \in \intPT{\int{\subclass}}$ \iff $(\neg d, \neg c) \in \intPT{\int{\subclass}}$;
 \end{enumerate}

 \item[Typing I:] \
 \begin{enumerate}
 \item $x \in \intCT{c}$ \iff $(x,c) \in \intPT{\int{\type}}$;

\item if $(p, c) \in \intPT{\int{\dom}}$ and $(x, y) \in \intPT{p}$ then $x \in \intCT{c}$;

 
 \item if $(p, c) \in \intPT{\int{\range}}$ and $(x, y) \in \intPT{p}$ then $y \in \intCT{c}$;

 \item if $(p, c) \in \intPT{\int{\dom}}$,  $x \in \intCF{c}$ and $y \in \down{\intPT{p}}$ then $(x, y) \in \intPF{p}$;

 \item if $(p, c) \in \intPT{\int{\range}}$,  $y \in \intCF{c}$ and $x \in \up{\intPT{p}}$ then $(x, y) \in \intPF{p}$;
 \end{enumerate}
 
 \item[Typing II:] \
 \begin{enumerate}
 \item For each $\eee \in \rhodfbotneg$, $\int{\eee} \in \Delta_{\DP}$;
 \item if $(p, c) \in \intPT{\int{\dom}}$ then $p \in \Delta_{\DP}$ and $c \in \Delta_{\DC}$;
 \item if $(p, c) \in \intPT{\int{\range}}$ then $p \in \Delta_{\DP}$ and $c \in \Delta_{\DC}$;
 \item if $(x, c) \in \intPT{\int{\type}}$ then $c \in \Delta_{\DC}$;
 \end{enumerate} 

\item[Disjointness I:] \   
 \begin{enumerate}
 \item if $(c, d) \in \intPT{\int{\disjC}}$  then $c,d\in\Delta_{\DC}$;
 
 \item if $(p, q) \in \intPT{\int{\disjP}}$  then $p,q\in\Delta_{\DP}$;
 
 \item $\intPT{\int{\disjC}}$ is symmetric, sub-transitive and exhaustive over $\Delta_{\DC}$;
 \begin{description}
  \item[Symmetry:]  if $(c,d)\in \intPT{\int{\disjC}}$, then $(d,c)\in \intPT{\int{\disjC}}$;
  \item[Sub-Transitivity:] if $(c,d)\in \intPT{\int{\disjC}}$ and $(e,c)\in\intPT{\int{\subclass}}$, then $(e,d)\in \intPT{\int{\disjC}}$;
  \item[Exhaustive:] if $(c,c)\in \intPT{\int{\disjC}}$ and $d\in\Delta_{\DC}$ then $(c,d)\in \intPT{\int{\disjC}}$; 
 
 \end{description}

 \item $\intPT{\int{\disjP}}$ is symmetric, sub-transitive and  exhaustive over $\Delta_{\DP}$;
  \begin{description}
    \item[Symmetry:] If $(p,q)\in \intPT{\int{\disjP}}$, then $(q,p)\in \intPT{\int{\disjP}}$;
    \item[Sub-Transitivity:]  if $(p,q)\in \intPT{\int{\disjP}}$ and $(r,p)\in\intPT{\int{\spp}}$, then $(r,q)\in \intPT{\int{\disjP}}$;
    \item[Exhaustive:] if $(p,p)\in \intPT{\int{\disjP}}$ and $q\in\Delta_{\DP}$ then $(p,q)\in \intPT{\int{\disjP}}$;
\end{description}

  \end{enumerate}

  \item[Disjointness II:] \   
   \begin{enumerate}
 \item if $(p, c) \in \intPT{\int{\dom}}$, $(q, d) \in \intPT{\int{\dom}}$, and $(c, d) \in \intPT{\int{\disjC}}$, then $(p, q) \in \intPT{\int{\disjP}}$;
 
  \item if $(p, c) \in \intPT{\int{\range}}$, $(q, d) \in \intPT{\int{\range}}$, and $(c, d) \in \intPT{\int{\disjC}}$, then $(p, q) \in \intPT{\int{\disjP}}$;
 
 \item $(c, d) \in \intPT{\int{\disjC}}$ \iff $(c, \neg d) \in \intPT{\int{\subclass}}$;
 
 \item $(p, q) \in \intPT{\int{\disjP}}$ \iff $(p, \neg q) \in \intPT{\int{\spp}}$.
 

  \end{enumerate}

\end{description}

\nd A graph $G$ is \emph{$\rhodfbotneg$-satisfiable} if it has a $\rhodfbotneg$-model $\I$. Moreover, given two $\rhodfbotneg$-graphs $G$ and $H$, we say that $G$ \emph{$\rhodfbotneg$-entails} $H$, denoted $G \rdfentbotneg H$, \iff every $\rhodfbotneg$-model of $G$ is also a $\rhodfbotneg$-model of $H$.
\end{definition}

\nd In the following, if clear from the context, for ease of presentation, we will omit the prefix $\rhodfbotneg$-.

\begin{remark}[About the semantics] \label{satbotneg}
Concerning Definition~\ref{satisfactionNew}, let us note the following:
\begin{enumerate}
    \item the positive extension of a negated class is the negative extension of that class; 
    
    \item by construction, we also have that for $\triple{s, \neg p, o} \in G$, $(\int{s}, \int{o}) \in \intPT{\int{(\neg p)}} = \intPT{\neg~\int{p}} = \intPF{\int{p}}$. That is, $\triple{s, \neg p, o} \in G$ states that ``$(s,o)$ belongs to the negative extension of $p$, \ie~$s$ has a non $p$ that is an $o$";
    
    
    \item by construction, \eg~if $(c, d) \in \intPT{\int{ \subclass}}$ then also  $\intCF{d} \subseteq \intCF{c}$;
    
    %
    
    \item by construction,  $x \in \intCF{c}$ \iff $(x, \neg c) \in \intPT{\int{\type}}$;

    \item concerning \eg~point 5 in Typing I,  from a FOL perspective, the reading of 
    $\triple{p,\range, c}$ is $\forall x. \forall y. p(x,y) \impf  c(y)$ and $\forall x. \forall y \in \down{p}. \neg c(y) \impf  \neg p(x,y)$. In the latter case, the aim is to limit the universal quantification in a reasonable way.
    
    
    %
    %
    
    
    
    \item the presence of \eg~$\triple{a,\type,b},\triple{a,\type,c}$ and $ \triple{b,\disjC,c}$ in a graph does not preclude its satisfiability. In fact, a $\rhodfbotneg$ graph will always be satisfiable (see Corollary~\ref{satrhodfbotneg} later on) avoiding, thus, the \emph{ex falso quodlibet} principle. This is in line with the $\rhodf$ semantics~\cite{Munoz09}.

    
\end{enumerate}

\end{remark}

\begin{example} \label{exrdfsII}
Consider Example~\ref{ex_rdfbotnegI}. Then, it may be verified that 
\begin{align}
G \rdfentbotneg \{ & \triple{\mathtt{brainTumor}, \mathtt{hasTreatment}, x},  \nonumber\\
 & \triple{x, \type, \neg\mathtt{antipyretic}} \}    \label{dedpara} \\
G \rdfentbotneg &  \triple{\mathtt{ebola}, \neg\mathtt{hasTreatment} , \mathtt{paracetomol}} \\
G \not\rdfentbotneg &  \triple{\mathtt{ebola}, \neg\mathtt{hasTreatment} , \mathtt{ebola}} \ .
\end{align}
\nd Note that the last entailment does not hold as ``ebola is not a treatment", which instead would hold without the restriction on the domain of the universal quantification.
\end{example}


\begin{remark} \label{nostat}
As anticipated in the related work section, \cite{Darari15} considers expressions of the form 
\begin{equation} \label{Noex}
No(\{\triple{s_1,p_1,o_1}, \ldots, \triple{s_n,p_n,o_n})
\end{equation}
 \nd in informal FOL terms $\neg \exists \vec{x}.(p_1(s_1,o_1) \land \ldots \land p_n(s_n,o_n))$, which is the same as, $\forall \vec{x}.(\neg p_1(s_1,o_1) \lor \ldots \lor \neg p_n(s_n,o_n))$,  where $\vec{x}$ are the variables occurring the triples. For instance,
\begin{equation} \label{obama}
 No(\{\triple{\mathtt{obama}, \mathtt{child}, x}, \triple{x,\mathtt{gender},\mathtt{male}}\})
\end{equation}
\nd  expresses that ``Obama has no son".



\rhodfbotneg~considers only the case $n=1$ in $(\ref{Noex})$ via the expression
$\triple{s, \neg p, \uniterm_c}$ as the general case would introduce a disjunction, which we would like to avoid for computational reasons. 
Nevertheless,  we may consider the option to use \eg
\[
\triple{\mathtt{obama}, \neg \mathtt{child}, \uniterm_\mathtt{malePerson}}
\]
\nd instead to express $(\ref{obama})$.


\end{remark}

\subsection{Deductive System for $\rhodfbotneg$} 

\nd We now present a deductive system for $\rhodfbotneg$.  

\begin{definition}[Deductive rules for $\rhodfbotneg$] \label{dedrhodfbotneg}
The \emph{deductive rules for $\rhodfbotneg$} are all rules for $\rhodf$ to which we add the following rules ($Z$ is an additional meta-variable): 
%
%
\begin{enumerate}
\item[2.] Subproperty: \\ [0.25em]
    \begin{tabular}{llll}
      $(c)$ & $\frac{\triple{A,  \spp, B}}{\triple{\neg B,  \spp, \neg A}}$ & \\ \\
      $(d)$ & $\frac{\triple{A, D, \uniterm_C}, \triple{D,\spp, E}}{\triple{A, E, \uniterm_C}}$ &
      $(e)$ & $\frac{\triple{\uniterm_C, D, A}, \triple{D,\spp, E}}{\triple{\uniterm_C, E, A }}$
    \end{tabular} 
    
\item[3.] Subclass: \\ [0.05em]
    \begin{tabular}{llll}
      $(c)$ & $\frac{\triple{A,  \subclass, B}}{\triple{\neg B,  \subclass, \neg A}}$ & \\ \\
      $(d)$ & $\frac{\triple{A, D, \uniterm_C}, \triple{B,\subclass, C}}{\triple{A, D, \uniterm_B}}$ &
      $(e)$ & $\frac{\triple{\uniterm_C, D, A}, \triple{B,\subclass, C}}{\triple{\uniterm_B, D, A }}$
    \end{tabular} 
 \item[4.] Typing:  \\ [0.25em]
    \begin{tabular}{llll}
      $(c)$ & $\frac{\triple{D, \dom, B},  \triple{X, \type, \neg B}, \triple{Z, D, Y}}{\triple{X, \neg D, Y}}$ & \\ \\
      $(d)$ & $\frac{\triple{D, \range, B},  \triple{Y, \type, \neg B}, \triple{X, D, Z}}{\triple{X, \neg D, Y}}$  \\ \\
     $(e)$ & $\frac{\triple{A, D, \uniterm_C}, \triple{Y, \type, C}}{\triple{A, D, Y}}$ 
     \hspace*{1em} $(f)$ \hspace*{1em} $\frac{\triple{\uniterm_C, D, B}, \triple{X, \type, C}}{\triple{X, D, B}}$ \\ \\
     
     $(g)$ & $\frac{\triple{A, D, \uniterm_C}, \triple{A, \neg D, Y} }{\triple{Y, \type, \neg C}}$ 
     \hspace*{1em} $(h)$ \hspace*{1em} $\frac{\triple{\uniterm_C, D, B}, \triple{X, \neg D, B} }{\triple{X, \type, \neg C}}$ \\
    \end{tabular} 

\item[6.] Conceptual Disjointness: \\[0.25em]
    \begin{tabular}{llllll}
      $(a)$ & $\frac{\triple{A,\disjC,B}}{\triple{B,\disjC,A}}$  & 
	 $(b)$ & $\frac{\triple{A,\disjC,B}, \triple{C,\subclass,A}}{\triple{C,\disjC,B}}$ & \\ \\
 $(c)$ & $\frac{\triple{A,\disjC,A}}{\triple{A,\disjC,B}}$ &
      $(d)$ & $\frac{\triple{A,\disjC,B}}{\triple{A,\subclass, \neg B}}$  & \\ \\
$(e)$ & $\frac{\triple{A,  \subclass, B}}{\triple{A, \disjC, \neg B}}$ 
      \end{tabular}   
          
\item[7.] Predicate Disjointness: \\[0.25em]
    \begin{tabular}{llllll}
      $(a)$ & $\frac{\triple{A,\disjP,B}}{\triple{B,\disjP,A}}$  &
 $(b)$ & $\frac{\triple{A,\disjP,B}, \triple{C,\spp,A}}{\triple{C,\disjP,B}}$ & \\ \\ 
 $(c)$ & $\frac{\triple{A,\disjP,A}}{\triple{A,\disjP,B}}$  &
      $(d)$ & $\frac{\triple{A,\disjP,B}}{\triple{A,\spp, \neg B}}$  \\ \\ 
 $(e)$ & $\frac{\triple{A,  \spp, B}}{\triple{A, \disjP, \neg B}}$ 
          \end{tabular}  
          
\item[8.] Crossed Disjointness: \\[0.25em]
    \begin{tabular}{llll}
      $(a)$ & $\frac{\triple{A,\dom,C}, \triple{B,\dom,D}, \triple{C,\disjC,D}}{\triple{A,\disjP,B}}$  & \\ \\ 
 $(b)$ & $\frac{\triple{A,\range,C}, \triple{B,\range,D}, \triple{C,\disjC,D}}{\triple{A,\disjP,B}}$\\
          \end{tabular} 
\end{enumerate}
\end{definition}
\nd Now, the definition of derivation among $\rhodfbotneg$-graphs $G$ and $H$, denoted $G \derivbotneg H$, is as for $\rhodf$ (see Definition~\ref{def:derivatioGn}), except that, of course, we now consider all rules of Definition~\ref{dedrhodfbotneg} instead.
Similarly, the \emph{$\rhodfbotneg$-closure} of a graph $G$, denoted $\closbotneg(G)$,  is defined as 
\[
\closbotneg(G) = \{\tau \mid G \derivbotneg^{*}~\tau\} \ ,
\]
\nd where $\derivbotneg^{*}$ is as $\derivbotneg$ except that rule $(1a)$ is excluded. 

\begin{example} \label{exrdfsD}
%
The following is a simple proof a $(\ref{dedpara})$:\\

\hspace*{-0.5cm}\begin{tabular}{lll}
$(1)$   & $\triple{\mathtt{opioid}, \disjC, \mathtt{antipyretic}}$ 
        & Rule $(1b)$ \\
$(2)$   & $\triple{\mathtt{opioid}, \subclass, \neg \mathtt{antipyretic}}$ 
        & Rule $(6d): (1)$ \\
$(3)$   & $\triple{\mathtt{morphine}, \type, \mathtt{opioid}}$ 
        & Rule $(1b)$ \\
$(4)$   & $\triple{\mathtt{morphine}, \type, \neg \mathtt{antipyretic}}$ 
        & Rule $(3b):(2),(3)$ \\
$(5)$   & $\triple{\mathtt{brainTumour}, \mathtt{hasDrugTreatment}, \mathtt{morphine}}$ 
        & Rule $(1b)$ \\
$(10)$  & $\triple{\mathtt{brainTumor}, \mathtt{hasTreatment}, x}$  \\
        & $\triple{x, \type, \neg\mathtt{antipyretic}}$ 
        & Rule $(1a)$: $(4), (5)$ \\
\end{tabular}
\end{example}

\nd In the following, we will also assume that the definition of entailment $\rdfent$ is extended naturally to $\rhodfbotneg$-graphs by considering $\disjC,\disjP, \neg p$ and $\uniterm_c$ as resources without any specific semantic constraint. In a similar way, we assume $\deriv$ (and $\clos(\cdot)$)  to be extended to $\rhodfbotneg$-graphs by assuming that triples involving $\disjC,\disjP, \neg p$ and $\uniterm_c$ are considered as $\rhodf$ triples. Then, the following can easily be proven:

\begin{proposition} \label{propbnII}
Let $G$ and $H$ be two $\rhodfbotneg$-graphs. Then, 
\begin{enumerate}
\item if $G \rdfent H$ then $G \rdfentbotneg H$;
\item if $G \deriv H$ then $G \derivbotneg H$;
\item $\clos(G) \subseteq \closbotneg(G)$.
\end{enumerate}
\end{proposition}

\nd Of course, conditions 1.-3. above do not hold in general for the opposite direction. For instance, for $G = \{\triple{a, \disjC, b}\}$ we have 
$G \nrdfent \triple{a, \subclass, \neg b}$, $G  \nderiv \triple{a, \subclass, \neg b}$ and $\triple{a, \subclass, \neg b} \notin \clos(G)$, but $G \rdfentbotneg \triple{a, \subclass, \neg b}$ and  $\triple{a, \subclass, \neg b} \in \closbotneg(G)$.

The following proposition defines the construction of the \emph{canonical model} for $\rhodfbotneg$ graphs and extends the result for $\rhodf$ (see Proposition~\ref{proprdf}), showing then that all $\rhodfbotneg$-graphs $G$ are satisfiable.

\begin{proposition}[$\rhodfbotneg$ Canonical model]\label{completefirstrhodfbotneg}
Given a $\rhodfbotneg$-graph $G$,  define  a $\rhodfbotneg$ \emph{interpretation} $\I_G$ as a tuple 
\[
\I_G =\tuple{\Delta_{\DR}, \Delta_{\DP}, \Delta_{\DC}, \Delta_{\DL}, \intPT{\cdot}, \intPF{\cdot}, \intCT{\cdot}, \intCF{\cdot},\intG{\cdot}}
\]

\nd such that: 
\begin{enumerate}

\item $\Delta_{\DR} \coloneqq \universe(G)\cup \{ \neg r \mid r \in \universe(G) \} \cup \ \rhodf$;




\item $\Delta_{\DP}' \coloneqq \{p\in \universe(G)\mid \text{either}~\triple{s,p,o}$, $\triple{s,p,\uniterm_c}$, $\triple{\uniterm_c,p,o}$, $\triple{p,\spp, q}$, $\triple{q,\spp, p}$, 
$\triple{p,\dom, c}$, $\triple{p,\range,d}$  or $\triple{p,\disjP,q}$ is in $\closbotneg(G) \} \ \cup \  \rhodfbotneg$;

\item $\Delta_{\DP} \coloneqq \Delta_{\DP}' \cup \{ \neg p \mid p \in \Delta_{\DP}'\}$;

\item $\Delta_{\DC}'\coloneqq\{c\in \universe(G)\mid \text{either}~\triple{x,\type, c}$, $\triple{c,\subclass, d}$, $\triple{d,\subclass, c}$, $\triple{p,\dom, c}$, $\triple{p,\range,c}$, 
$\triple{s,p,\uniterm_c}$, $\triple{\uniterm_c,p,o}$ or $\triple{c,\disjC, d}$ is in $\closbotneg(G) \}$; 

\item $\Delta_{\DC} \coloneqq \Delta_{\DC}' \cup \{ \neg c \mid c \in \Delta_{\DC}'\}$;

\item $\Delta_{\DL} \coloneqq \universe(G)\cap \AL $;

\item  $\intPT{\cdot}$ and $\intPF{\cdot}$ are extension functions $\Delta_{\DP} \to 2^{\Delta_{\DR} \times \Delta_{\DR}}$  s.t.~$\intPT{p}\coloneqq\{(s,o)\mid\triple{s,p,o}\in \closbotneg(G)\}$ and $\intPF{p}\coloneqq \intPT{\neg p}$;


%
	
\item $\intCT{\cdot}$ and $\intCF{\cdot}$ are extension functions $\Delta_{\DC} \to 2^{\Delta_{\DR}}$ s.t. $\intCT{c}\coloneqq\{x\in \universe(G)\mid\triple{x,\type,c}\in \closbotneg(G)\}$ and $\intCF{c}\coloneqq \intCT{\neg c}$;


\item $\intG{\cdot}$ is the identity function over $\Delta_{\DR}$.


\end{enumerate}

\nd Then, $\I_G\rdfsatbotneg G$.

\end{proposition}
\begin{proof}
We prove that $\I_G$ satisfies the constraints in Definition~\ref{satisfactionNew}. 
%
We illustrate here the proof of some of the conditions in Definition~\ref{satisfactionNew} only. The others can be worked out similarly.
\begin{description}\label{condRDFproof}
 \item[Simple:] \
 \begin{enumerate}
  \item[1.] Suppose $\triple{s, p, o} \in G$ and neither $s$ nor $o$ are of the form $\uniterm_c$. Then by construction $p \in\Delta_{\DP}$ and $(s, o) \in \intPT{p}$, which concludes.
  
  \item[2.] Suppose $\triple{s, p, \uniterm_c} \in G$. Then, by construction $p \in\Delta_{\DP}$ and $c \in \Delta_{\DC}$. Now, assume $y \in \intCT{c}$ and, thus, by construction  $\triple{y,\type,c}\in \closbotneg(G)$. Therefore, by rule $(4xe)$ we have also $\triple{s,p,y}\in \closbotneg(G)$ and, thus,
  $(s, y) \in \intPT{p}$ by construction, which concludes.
 
 \end{enumerate} 
 
\item[Subclass:] \
 \begin{enumerate}
 \item[2.] 
 %
 Assume $(c, d) \in \intPT{\subclass}$. By construction, $\triple{c,\subclass, d}\in \closbotneg(G)$ and $c, d \in \Delta_{\DC}$. Now, assume $x \in \intCT{c}$ and, thus, by construction  $\triple{x,\type,c}\in \closbotneg(G)$. Therefore, by rule $(3b)$ we have also $\triple{x,\type,d}\in \closbotneg(G)$  and, thus,
  $x \in \intCT{d}$ by construction. As a consequence, $\intCT{c} \subseteq \intCT{d}$. Eventually, assume 
  $x \in \intCF{d}$ and, thus, by construction both $x \in \intCT{\neg d}$ and $\triple{x,\type, \neg d}\in \closbotneg(G)$ hold. But, $\triple{c,\subclass, d}\in \closbotneg(G)$ implies, by rule $(3c)$, $\triple{\neg d,\subclass, \neg c}\in \closbotneg(G)$ and, thus, by rule $(3b)$ we have $\triple{x,\type, \neg c}\in \closbotneg(G)$. Therefore, by construction $x \in \intCT{\neg c} = \intCF{c}$ and, thus, 
  $\intCF{d} \subseteq \intCF{c}$, which concludes.
 \end{enumerate}

 \item[Typing I:] \
 \begin{enumerate}
 
 \item[2.]  
 Assume $(p, c) \in \intPT{\int{\dom}}$ and $(x, y) \in \intPT{p}$. By construction, both 
 $\triple{p,\dom, c}$ and $\triple{x,p, y}$ are in $\closbotneg(G)$. Therefore, by rule $(4a)$, 
 $\triple{x,\type, c} \in \closbotneg(G)$ and, thus, by construction $x \in \intCT{c}$, which concludes.
 
 \item[4.]   
 Assume $(p, c) \in \intPT{\int{\dom}}$,  $x \in \intCF{c}$ and $y \in \down{\intPT{p}}$. By construction, we have that  $\{ \triple{p,\dom, c}, \triple{x,\type, \neg c}, \triple{z,p, y} \} \subseteq \closbotneg(G)$.
 Therefore, by rule $(4c)$, $\triple{x, \neg p, y} \in \closbotneg(G)$ and, thus, by construction, 
 $(x, y) \in \intPT{\neg p} =  \intPF{p}$, which concludes.
 \end{enumerate}
 
 \item[Typing II:] \
 \begin{enumerate}
  \item[2.] 
  Assume $(p, c) \in \intPT{\int{\dom}}$. Then, by construction $\triple{p,\dom, c} \in \closbotneg(G)$ and, thus, $p \in\Delta'_{\DP} \subseteq \Delta_{\DP}$ and $c \in\Delta'_{\DC} \subseteq \Delta_{\DC}$, which concludes. 
 \end{enumerate}
 
 \item[Disjointness I:] \   
\begin{enumerate}
\item[3.] 
 \begin{description}
  \item[Symmetry:] 
Assume $(c,d)\in \intPT{\int{\disjC}}$. By construction, 
$\triple{c,\disjC, d} \in \closbotneg(G)$ and, thus, by rule $(4a)$ 
$\triple{d,\disjC, c} \in \closbotneg(G)$. Therefore, by construction, $(d,c)\in \intPT{\int{\disjC}}$, which concludes.
\end{description}
\end{enumerate}
 
\item[Disjointness II:] \   
\begin{enumerate}
 \item[1.] 
Assume $(p, c) \in \intPT{\int{\dom}}$, $(q, d) \in \intPT{\int{\dom}}$, and $(c, d) \in \intPT{\int{\disjC}}$. Then, by construction, we have that  $\{ \triple{p,\dom, c}, \triple{q,\dom, d}, \triple{c,\disjC, d} \} \subseteq \closbotneg(G)$ and, thus, by rule $(8a)$ 
$\triple{p,\disjP, q} \in \closbotneg(G)$. Therefore, by construction, $(p,q)\in \intPT{\int{\disjP}}$, which concludes.

 \item[3.] 
 If $(c, d) \in \intPT{\int{\disjC}}$ then, by construction, $\triple{c,\disjC, d} \in \closbotneg(G)$ 
 and, thus, by rule $(6d)$  $\triple{c,\subclass, \neg d} \in \closbotneg(G)$. Therefore, by construction, $(c, \neg d) \in \intPT{\int{\subclass}}$. Vice-versa, if $(c, \neg d) \in \intPT{\int{\subclass}}$ then 
 by construction, $\triple{c,\subclass, \neg d} \in \closbotneg(G)$ and, thus, by rule $(6e)$  $\triple{c,\disjC, d} \in \closbotneg(G)$. Therefore, by construction, $(c, d) \in \intPT{\int{\disjC}}$, which concludes.
\end{enumerate}
 \end{description}
\end{proof}
\nd By Proposition~\ref{completefirstrhodfbotneg}, it follows that

\begin{corollary} \label{satrhodfbotneg}
Every $\rhodfbotneg$-graph is satisfiable.
\end{corollary}


\nd We prove now soundness and completeness  of our deduction system for $\rhodfbotneg$.
The proofs are inspired by the analogous ones in~\cite{Munoz09} for $\rhodf$.

The following proposition is needed for soundness.

\begin{proposition}[Soundness]\label{sound}
Let $G$ and $H$ be $\rhodfbotneg$-graphs and let one of the following statements hold:
\begin{enumerate}
\item there is a map $\mu:H\rightarrow G$;
\item $H\subseteq G$;
\item there is an instantiation $R/R'$ of one of the rules in Definition~\ref{dedrhodfbotneg}, such that
$R \subseteq G$ and $H = G\cup R'$.
\end{enumerate}

\nd Then, $G\rdfentbotneg H$.
\end{proposition}
\begin{proof}
%
By Corollary~\ref{satrhodfbotneg} we know that $G$ is satisfiable. So, let 
 \[
\I =\tuple{\Delta_{\DR}, \Delta_{\DP}, \Delta_{\DC}, \Delta_{\DL}, \intPT{\cdot}, \intPF{\cdot}, \intCT{\cdot}, \intCF{\cdot},\int{\cdot}} \ 
\]

\nd be a model of $G$, \ie~$\I\rdfsatbotneg G$. Therefore, $\I $  satisfies the constraints in Definition~\ref{satisfactionNew}. We have to prove that $\I\rdfsatbotneg H$. The proof is split in cases depending on rule applications of which we address here only some of them. 
The other cases can be shown similarly. 

\begin{description}
\item[Rule $(4e)$.]
Let $\{\triple{s,p, \uniterm_c}, \triple{o,\type, c}\} \subseteq R$ for some $R\subseteq G$, $R'= \{\triple{s,p, o}\}$, obtained via the application of rule $(4e)$, and $H=G\cup R'$.
As $\I\rdfsatbotneg G$ and $R\subseteq G$, $\I\rdfsatbotneg R$ follows. Therefore, $\I \rdfsatbotneg \triple{o,\type, c}$ and, thus, $\int{o} \in \intCT{\int{c}}$ follows. But, also $\I \rdfsatbotneg \triple{s,p, \uniterm_c}$ and, thus, by condition Simple, case 2.~in Definition~\ref{satisfactionNew}, we have that 
$(\int{s}, \int{o}) \in \intPT{\int{p}}$. That is, $\I \rdfsatbotneg \triple{s,p, o}$. Hence, from  $\I\rdfsatbotneg R'$, $\I\rdfsatbotneg G$ and $H = G\cup R'$, $\I\rdfsatbotneg H$ follows.


\item[Rule $(6d)$.]
Let $\triple{c,\disjC,d}\in R$ for some $R\subseteq G$, $R'= \{\triple{c,\subclass, \neg d}\}$, obtained via the application of rule $(6d)$, and $H=G\cup R'$.
As $\I\rdfsatbotneg G$ and $R\subseteq G$, $\I\rdfsatbotneg R$ follows.  Therefore, 
$\I \rdfsatbotneg \triple{c,\disjC,d}$ and, thus, $(\int{c},\int{d})\in\intPT{\int{\disjC}}$.
But, by condition Disjointness II, case 3.~in Definition~\ref{satisfactionNew}  we have that
$(\int{c},\neg \int{d})\in\intPT{\int{\subclass}}$ and, thus, $(\int{c}, \int{(\neg d)})\in\intPT{\int{\subclass}}$. 
Therefore,  $\I \rdfsatbotneg \triple{c,\subclass, \neg d}$. Hence, from  $\I\rdfsatbotneg R'$, $\I\rdfsatbotneg G$ and $H = G\cup R'$, $\I\rdfsatbotneg H$ follows.
\end{description}
\end{proof}

\begin{proposition}\label{completesecond}
Let $G$ and $H$ be $\rhodfbotneg$-graphs. If $G\rdfentbotneg H$ then there is a map 
$\mu: H\rightarrow \closbotneg(G)$.
\end{proposition}
\begin{proof}
Consider the canonical model 
\[
\I_G =\tuple{\Delta_{\DR}, \Delta_{\DP}, \Delta_{\DC}, \Delta_{\DL}, \intPT{\cdot}, \intPF{\cdot}, \intCT{\cdot}, \intCF{\cdot},\intG{\cdot}}
\]
\nd of $G$, as defined in Proposition~\ref{completefirstrhodfbotneg}.
As $G\rdfentbotneg H$, $\I_G\rdfsatbotneg H$ follows.
Therefore,  for each $\triple{s,p,o}\in H$, $\intG{p} \in\Delta_\DP$
and $(\intG{s}, \intG{o})\in \intPT{\intG{p}}$.  By construction, $\intG{p} = p$, and 
$\intPT{\intG{p}} = \intPT{p} =  \{(t, t')\mid \triple{t, p, t'}\in\closbotneg(G)\}$.
Finally, since $(\intG{s},\intG{o})\in \intPT{p}$, we have that $\triple{\intG{s},p,\intG{o}}\in \closbotneg(G)$, \ie~$\triple{\intG{s},\intG{p},\intG{o}}\in \closbotneg(G)$. Therefore, $\intG{\cdot}$ is a map such that $\intG{H} \subseteq \closbotneg(G)$, \ie~a map $\intG{\cdot}\colon H\rightarrow \closbotneg(G)$, which concludes.
\end{proof}

\nd From Proposition~\ref{completesecond} we get immediately the following corollary:

\begin{corollary} \label{colbotneg}
Let $G$ and $H$ be $\rhodfbotneg$-graphs. If  $G \rdfentbotneg H$ then there is a proof of $H$ from $G$ where rule $(1a)$ is used at most once and at the end.
 \end{corollary}

\nd Eventually, combining previous Propositions~\ref{sound} and~\ref{completesecond}, we get soundness and completeness of our deductive system.

\begin{theorem}[Soundness \& Completeness]\label{Th:soundcompletebotneg}
Let $G$ and $H$ be $\rhodfbotneg$-graphs. Then
$G\derivbotneg H$ iff $G\rdfentbotneg H$.
\end{theorem}
\begin{proof}
Concerning soundness, if $G\derivbotneg H$ then, by Proposition~\ref{sound},  $G\rdfentbotneg H$. Concerning completeness, if $G\rdfentbotneg H$ then, by Proposition~\ref{completesecond}, $H$ can be obtained from $\closbotneg(G)$ using rule (1a). Therefore, as $G\derivbotneg \closbotneg(G)$, $G\derivbotneg H$ follows, which concludes. 
\end{proof}



\nd Finally, unlike $\rhodf$ (see Proposition~\ref{proprdf}), the size of the closure of a \rhodfbotneg~graph $G$ is $\Theta(|G|^3)$. The upper bound comes from the fact that in a triple $\triple{s,p,o}$, for each $s,p$ and $o$ we may have at most $|G|$ terms, while the lower bound is given by the following example.
\begin{example}\label{n3}
It can easily be verified that for $1 \leq i<j \leq n$ and $1 \leq l,k,h \leq n$
\[
\{\triple{a_i, \type, c}, \triple{a_i,p_1,\uniterm_c}, \triple{p_i, \spp, p_j}  \} \derivbotneg \triple{a_l, p_k, a_h}  \ ,
\]
\nd and, thus, the number of triples in the closure is $\Omega(|G|^3)$.
\end{example}

\nd Furthermore, it is not difficult to see that if $\uniterm_c$ terms do not occur in a \rhodfbotneg~graph $G$, then 
the closure of $G$ remains quadratically upper bounded. As case \ii{i} in Remark~\ref{np} also applies to \rhodfbotneg, it can be shown that

\begin{proposition} \label{proprdfbotneg}
Let $G$ and $H$ be $\rhodfbotneg$-graphs. Then
\begin{enumerate}
\item the closure of $G$ is unique and $|\closbotneg(G)| \in \Theta(|G|^3)$;
\item if $\uniterm_c$ terms do not occur in $G$ then  $|\closbotneg(G)| \in \Theta(|G|^2)$;
\item deciding $G\rdfentbotneg H$ is an NP-complete problem;
\item if $G$ is ground then $\closbotneg(G)$ can be determined without using implicit typing rules $(5)$;

\item if $H$ is ground, then  $G\rdfentbotneg H$ \iff $H \subseteq \closbotneg(G)$;


\item There is no triple $\tau$ such that $\emptyset \models \tau$.
\end{enumerate}
 
\end{proposition}

\nd Eventually, by Proposition~\ref{proprdfbotneg} it follows immediately that
\begin{corollary} \label{propnlognbotneg}
Let $G$ and $H$ be two ground \rhodfbotneg~graphs. Then deciding if $G \rdfentbotneg H$ can be done in time $O(|H||G|^3)$ and in time $O(|H||G|^2)$ if $\uniterm_c$ terms do not occur in $G$.
\end{corollary}


\section{Conclusions} \label{concl}

We have addressed the problem to add negative statements of various form considered as relevant for RDFS by the literature. We have presented a sound and complete deductive system that consists of RDFS rules plus some additional rules to deal with the extra type of triples we allow.  The design of the semantics has been such that to preserve features such as the canonical model property and computational attractiveness. 

As future work, Corollary~\ref{propnlognbotneg} tells us that there is still some computational complexity gap \wrt~\rhodf (see Proposition~\ref{propnlogn}), which we would like to reduce as much as possible. In particular, we are going to investigate whether the principles of the method proposed in~\cite{Munoz09} for \rhodf~can be adapted to \rhodfbotneg~as well. Additionally, we would like to address query answering, in particular to extend our framework to SPARQL and to verify whether and how it impacts \wrt~\rhodfbotneg~graphs.

\newpage

\section*{Acknowledgments}
\nd This research was  partially supported by TAILOR (Foundations of Trustworthy AI – Integrating Reasoning, Learning and Optimization), a project funded by EU Horizon 2020 research and innovation programme under GA No 952215.


\end{document}